%% file: main.tex
\documentclass[10pt]{article} % For LaTeX2e
\usepackage[preprint]{rlc}

\usepackage{amssymb}            % Defines common symbols like \mathbb R
\usepackage{mathtools}          % Extends amsmath, providing common math tools
\usepackage{amsthm}
\usepackage{mathrsfs}           % Enables \mathscr, which can work in cases that \mathcal does not
\usepackage{graphicx}           % For including images
\usepackage{subcaption}         % Allows for the use of subfigures and subcaptions
\usepackage[space]{grffile}     % For spaces in image names
\usepackage{url}                % For displaying urls
\usepackage{tikz}
\usetikzlibrary{arrows.meta,automata,positioning,calc}
\usepackage{wrapfig}
\usepackage{bbm}
\usepackage{booktabs} % for professional tables
\usepackage{multirow}
\usepackage{float}
\usepackage{array}
\usepackage{makecell}
\usepackage[export]{adjustbox}
\usepackage{xspace}
\makeatletter
\DeclareRobustCommand\onedot{\futurelet\@let@token\@onedot}
\def\@onedot{\ifx\@let@token.\else.\null\fi\xspace}

\makeatother

\newtheorem{thm}{Theorem}

\newtheorem{defn}[thm]{Definition}

\newtheorem{prop}[thm]{Proposition}

\title{Reinforcement Learning from Delayed Observations via World Models}

\author{\name Armin Karamzade \\
    karamzaa@uci.edu\\
    \addr Department of Computer Science\\
    \addr University of California, Irvine
    \And
    Kyungmin Kim  \\
    kyungk7@uci.edu \\
    \addr Department of Computer Science\\
    \addr University of California, Irvine
    \And
    Montek Kalsi  \\
    kalsim@uci.edu \\
    \addr Department of Computer Science\\
    \addr University of California, Irvine
    \And
    Roy Fox  \\
    royf@uci.edu \\
    \addr Department of Computer Science\\
    \addr University of California, Irvine
    \And ~~~~~~~~~~~~~~~~~~~~~~~~~~~~~~~~~~~~~~~~\\
    }

\begin{document}

\maketitle

\begin{abstract}
In standard reinforcement learning  settings, agents typically assume immediate feedback about the effects of their actions  after taking them. However, in practice, this assumption may not hold true due to physical constraints and can significantly impact the performance of learning algorithms. In this paper, we address observation delays in partially observable environments. We propose leveraging world models, which have shown success in integrating past observations and learning dynamics, to handle observation delays. By reducing delayed POMDPs to delayed MDPs with world models, our methods can effectively handle partial observability, where existing approaches achieve sub-optimal performance or degrade quickly as observability decreases. Experiments suggest that one of our methods can outperform a naive model-based approach by up to $250\%$. Moreover, we evaluate our methods on visual delayed environments, for the first time showcasing delay-aware reinforcement learning continuous control with visual observations.
\end{abstract}

\input{sec/intro}
\input{sec/related_work}
\input{sec/preliminaries}

\input{sec/method}

\input{sec/experiments}
\input{sec/discussion_conclusion}

\bibliography{main}
\bibliographystyle{rlc}

\newpage

\appendix
\input{sec/appendix}

\end{document}

%% file: sec/intro.tex
\section{Introduction}

Reinforcement Learning (RL) has emerged as a powerful framework for training agents to make sequential decisions in their environment. In traditional RL settings, agents assume immediate observational feedback from the environment about the effect of their actions. However, in many real-world applications, observations are delayed due to physical or technological constraints on sensors and communication, challenging this fundamental assumption. Delay can arise from various sources, such as computational limitations~\citep{dulac2019challenges}, communication latency and interconnection~\citep{ge2013modeling, rostami2023federated}, or physical constraints in robotic systems~\citep{imaida2004ground}.

For example, drone navigation based on computation offloading might experience lag when the network is congested~\citep{almutairi2022delay}, or robots equipped with shielding may encounter delays in execution to ensure the safe behavior of the policy ~\citep{corsi2024verification}. In scenarios where timely decision-making is critical and agents cannot afford to wait for updated state observations, RL algorithms must nonetheless find effective control policies subject to delay constraints. In this paper, we focus on observation delays that prevent the agent from immediately perceiving world state transitions, rather than execution delays that prevent the immediate application of the agent's control action, although these types of delay are interconnected and can in some settings be effectively addressed within a unified framework~\citep{katsikopoulos2003markov}. Specifically, at time $t$, the agent receives observation $o_{t-d}$ and reward $r_{t-d}$, where $d$ is the time delay.

The body of work on RL with delay has explored several approaches within the Markov Decision Process (MDP) framework. Memoryless approaches build a policy based on the last observed state~\citep{schuitema2010control}. A second type of approach aims to reduce the problem into an undelayed MDP by extending the states with additional information, typically the actions taken since the last available observation \citep{walsh2007planning, derman2020acting}. Finally, recent approaches compute, from the extended state, perceptual features predictive of the hidden current state to inform action selection~\citep{chen2021delay, liotet2021learning, liotet2022delayed, wang2024addressing}. While there are many existing works on delays in MDPs, surprisingly few study delays in Partially Observable MDPs (POMDPs) where the delayed observations are non-Markov~\citep{kim1987partially, varakantham2012delayed}, and these works do not provide a learning paradigm.

World models have recently shown significant success in integrating past observations and learning the dynamics of the environment~\citep{ha2018world}. These models, comprising a representation of the environment's state, a transition model depicting state evolution over time, and an observation model linking states to observations, have proven effective in capturing intricate temporal dependencies and enhancing decision-making. One such family is Dreamer~\citep{hafner2023dreamerv3}, a model-based RL framework that trains the agent through trajectories simulated by a learned world model. Dreamer benefits from the sample efficiency inherent in model-based RL techniques and is relatively insensitive to task-specific hyperparameter tuning.

In this paper, we propose leveraging world models to learn in the face of observation delays. We employ world models to form the extended state in the latent space, demonstrating that this latent extended state contains sufficient information for the current delayed state. This suggests two different strategies for adapting world models to POMDPs with observation delay: either by directly modifying the policy or by predicting the delayed latent state with imagination. While naively using world models for delays can lead to significant performance degradation as the delay increases, our methods exhibit greater resilience and one of them improves policy value upon the naive baseline by approximately $250\%$. Despite their simple implementation, these modifications achieve better or comparable performance to other approaches without the need for domain-specific hyperparameter tuning. Moreover, we evaluate our methods not only on vector inputs, but on continuous control tasks with visual inputs, a crucial aspect that was missing in the delayed RL community.

\textbf{Contributions} of this paper are summarized as follows:
\begin{itemize}
    \item We propose three methods that use world models to address observation delays. As a case study, we have adapted Dreamer-V3 to evaluate the effectiveness of our proposed strategies.
    \item We formalize observation delays in POMDPs and establish a link between delays in MDPs and POMDPs.
    \item We conduct extensive experiments that, among other domains, benchmark for the first time delayed RL in visual domains that are inherently partially observable.
\end{itemize}

%% file: sec/related_work.tex
\section{Related Work}

Several prior works have addressed delays in the MDP framework. One line of research employs a memoryless approach, where an agent uses only the last available state as input. For instance, \cite{schuitema2010control} proposed dSARSA, a memoryless extension of SARSA{~\citep{sutton2018reinforcement}} for delays in MDPs. Although this approach perceives the environment with partial information, it was shown to work quite well in some domains.

Another approach is to reduce a delayed MDP, which is known to be a structured POMDP, to an extended-state MDP by augmenting the recently observed state with the actions that have been taken since then~\citep{walsh2007planning,derman2020acting,bouteiller2020reinforcement}. For instance, DCAC~\citep{bouteiller2020reinforcement} extends SAC \citep{haarnoja2018soft} to take the extended state and achieves good sample efficiency through resampling techniques, but it suffers from an exponentially growing input dimension as the delay increases.

Recent strategies focus on deriving useful features from the extended state for policy input. \mbox{\cite{walsh2009learning}} developed a deterministic dynamics model to predict the unobserved state. \mbox{\cite{chen2021delay}} employed a particle-based method to simulate potential current state outcomes. Similarly, D-TRPO~\citep{liotet2021learning} obtains a belief representation of the current state using a normalizing flow, enhancing policy input with these features.

More recently, \cite{liotet2022delayed} applied imitation learning to train a delayed agent using an expert policy from an undelayed environment, though this approach is constrained by the need for access to the undelayed environment, limiting its practicality. Concurrently with our work, \mbox{\cite{wang2024addressing}} introduces a method for delay-reconciled training that integrates a critic and an extended-state actor and \cite{valensi2024tree} uses  
EfficientZero \citep{NEURIPS2021_d5eca8dc} for inferring future states, similarly to one of our methods.

%% file: sec/preliminaries.tex
\section{Preliminaries}

\begin{figure}[t]
    \centering
    \begin{subfigure}[b]{0.45\textwidth}
        \centering
        \includegraphics[width=\textwidth]{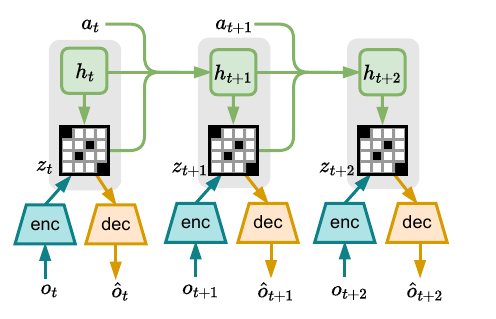}
        \caption{World-Model Learning}
        \label{fig:world-model}
    \end{subfigure}
    \begin{subfigure}[b]{0.45\textwidth}
        \centering
        \includegraphics[width=\textwidth]{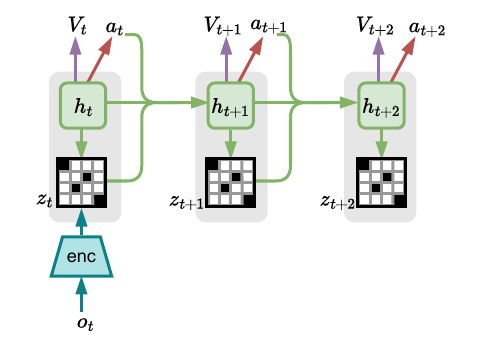}
        \caption{Actor-Critic Learning}
        \label{fig:actor-critic}
    \end{subfigure}
    
    \vspace{2em}
    \begin{subfigure}[b]{0.49\textwidth}
        \centering
        \includegraphics[width=\textwidth]{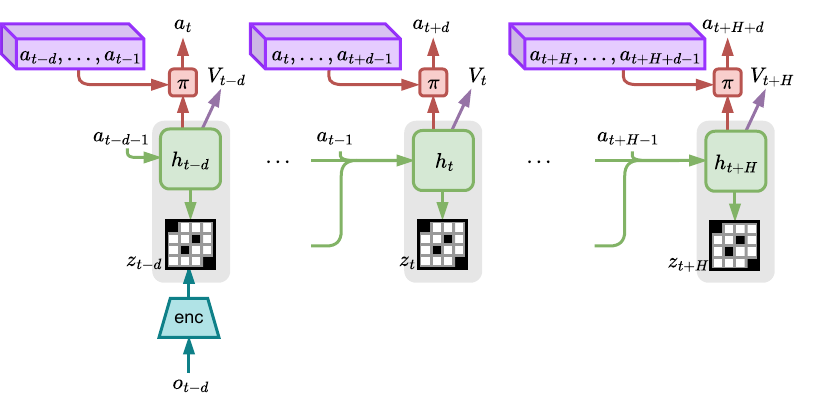}
        \caption{Extended Actor}
        \label{fig:extended-diagram}
    \end{subfigure}
    \begin{subfigure}[b]{0.43\textwidth}
        \centering
        \includegraphics[width=\textwidth]{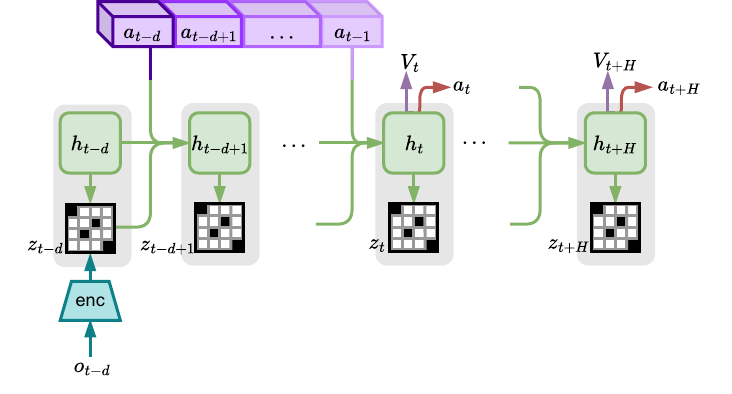}
        \caption{Latent State Imagination}
        \label{fig:latent-diagram}
    \end{subfigure}
    \caption{
    Panels (\subref{fig:world-model}) and (\subref{fig:actor-critic}) depict the standard Dreamer learning process, while (\subref{fig:extended-diagram}) and (\subref{fig:latent-diagram}) illustrate two strategies for adapting Dreamer for observation delays. (see section \ref{sec:delay-aware} and \ref{sec:delayed-ac})
    }
\end{figure}

\subsection{Delayed POMDPs} \label{subsec:DPOMDP}

A Partially Observable Markov Decision Process (POMDP) consists of a tuple $\langle\mathcal{S}, \mathcal{A}, \mathcal{T}, \Omega, \mathcal{O}, \gamma\rangle$, where $\mathcal{S}$ and $\mathcal{A}$ are the sets of states and actions, $\mathcal{T}(s', r|s, a)$ is the joint probability distribution of the next state and reward, $\Omega$ is the set of observations, $\mathcal{O}(o|s)$ is the conditional probability of observing $o\in \Omega$ in state $s$, and $\gamma \in[0, 1)$ is the discount factor. At each timestep $t$, in state $s_t$ the agent receives an observation $o_t \sim \mathcal{O}(o_t|s_t)$ and selects an action $a_t$, causing a transition and reward with distribution $\mathcal{T}(s_{t+1}, r_t|s_t, a_t)$. The goal is to maximize the expected return $\mathbb{E}[\sum_{t \geq 0}\gamma^t r_t]$. 

We define a Delayed Partially Observable Markov Decision Process (DPOMDP) as a tuple \(\langle \mathcal{P}_o, \mathcal{D} \rangle\), where \(\mathcal{P}_o\) represents a POMDP and \(\mathcal{D}\) is a distribution dictating the delay in receiving observations from \(\mathcal{P}_o\). Specifically, at time \(t\), the agent receives observation \(o_{t-d}\) and reward \(r_{t-d}\), where \(d \sim \mathcal{D}\). The objective is to maximize the expected return while operating in \(\mathcal{P}_o\) under these delays. For simplicity, we assume in this work that \(\mathcal{D}\) generates a constant non-negative integer delay \(d\), though our methods can be easily generalized to handle random delays as well.

In a similar fashion, delayed Markov Decision Processes (DMDPs) are defined in previous works~\citep{wang2024addressing, liotet2022delayed} as a tuple $\langle\mathcal{M}, \mathcal{D}\rangle$, where $\mathcal{M} = \langle\mathcal{S}, \mathcal{A}, \mathcal{T}, \gamma\rangle$ is an MDP. Like their undelayed counterparts, DMDPs can be considered a special case of DPOMDPs in which the (delayed) observation gives full information about the state at that time. Any DMDP can be reduced to an extended MDP $\widetilde{\mathcal{M}}$ by defining the states as a concatenation of the last observed state and the subsequent actions, which the agent needs to explicitly remember~\citep{altman1992closed}. In particular, $\widetilde{\mathcal{M}} = \langle\mathcal{X}, \mathcal{A}, \widetilde{\mathcal{T}},\gamma\rangle$, where $\mathcal{X} = \mathcal{S} \times \mathcal{A}^{d}$ and for $x=\left(s, a^1, \ldots, a^d\right), x'=\left(s', a'^1, \ldots, a'^{d}\right) \in \mathcal{X}$,
\begin{align}
    \widetilde{\mathcal{T}}(x', r | x, a) = \mathcal{T}(s', r | s, a^1) \mathbbm{1}([a'^1, \ldots, a'^d] = [a^2, \ldots, a^d, a]).
\end{align}

By definition, any DPOMDP is another POMDP with a specific structure. It can be formally constructed by defining the set of states in the new PODMP as the concatenation of the last $d+1$ states, of which only the earliest one is observable, and defining the transition probabilities accordingly. Thus, we can employ the POMDP framework to tackle the observation delays by constructing the equivalent POMDP~\citep{varakantham2012delayed}. On the other hand, one can exploit the structure introduced by delay in the process. Specifically, delay in receiving observations is equivalent to delay in inferring the latent states. As latent space enjoys the Markov property, we can then work with the DMDP defined over latent states and induced by delay in observations. Proposition \ref{prop} in section \ref{sec:prop} formalizes this intuition. Note, however, that there is a tradeoff between these two choices: on the one hand, POMDPs are harder to learn than MDPs; on the other hand, in the POMDP formulation, the world remembers the history for us, avoiding the curse of dimensionality in explicit agent context.

\subsection{World models}

World models~\citep{schmidhuber2015learning,ha2018world} simulate aspects of the environment by learning an internal representation through an encoder and a dynamics model. The encoder compresses high-dimensional inputs, such as image observations, into a lower-dimensional embedding, while the dynamics model forecasts future states from historical information. This streamlined state representation then serves as input to an RL agent.

\textbf{Dreamer}~\citep{hafner2019dream,hafner2020mastering,hafner2023dreamerv3}, a pioneering model-based RL (MBRL) approach utilizing world models, surpasses many model-free RL algorithms in data efficiency and performance. Dreamer's training involves three alternating phases: 1) training the world model on past experiences; 2) learning behaviors with actor–critic algorithm through imagined sequences; and 3) collecting data in the environment.

Figure \ref{fig:world-model} depicts the world model comprising an encoder-decoder and a Recurrent State Space Model (RSSM)~\citep{hafner2019learning} to model dynamics. Alongside reconstructed observation \(\hat{o}_t\), the model includes prediction heads for reward and episode continuation, omitted for clarity. The model state \(m_t = [h_t, z_t]\) comprises deterministic and stochastic components, respectively, given by
\begin{align}
    &h_t = f_\phi(h_{t-1}, z_{t-1}, a_{t-1}) \label{eq:h}\\
    &z_t \sim q_\phi(z_t | o_t, h_t). \label{eq:z}
\end{align}
Dreamer also learns a dynamic predictor for estimating imagined latent states as
\begin{equation}
    z'_t \sim p_\phi(z'_t|h_t) \label{eq:im}.
\end{equation}

As \eqref{eq:h} and \eqref{eq:z} suggest, the latent states are constructed to have a Markov structure, regardless of whether the observations themselves are Markov. In other words, Dreamer attempts to learn a latent-state MDP that is as nearly as possible equivalent to the observed POMDP. If it succeeds, this equivalence means that, for any policy $\pi(a | m)$, the stochastic process $\{m_t, a_t, r_t\}_{t\ge0}$ induced by the policy in the world model (Figure \ref{fig:actor-critic}) has the same joint distribution as the embedding of real policy rollouts (Figure \ref{fig:world-model}). Figure \ref{fig:actor-critic} also illustrates Dreamer's behavior learning using an actor--critic method, policy and value heads on latent trajectories predicted by the world model.

\subsection{Hardness of delayed control}

\begin{wrapfigure}{R}{0.28\textwidth}
    \vspace{-1.75em}
    \centering
    \begin{tikzpicture}[scale=0.5,->,>=Stealth,shorten >=1pt,auto,node distance=2cm,
                        thick,main node/.style={circle,draw,minimum size=1cm,inner sep=0pt]}]
      \node[main node] (s1) {$s_1$};
      \node[main node] (s2) [right=of s1] {$s_2$};
      \node[main node] (s3) [below=of $(s1)!0.5!(s2)$] {$s_3$};
      \path[every node/.style={font=\sffamily\small}]
        (s1) edge [loop above] node {$a_1, 1-\delta$} (s1)
             edge [bend left] node {$a_1, \delta$} (s2)
             edge [bend right] node[swap] {$a_2, 1$} (s3)
        (s2) edge [loop above] node {$a_2, \delta$} (s2)
              edge [bend left] node {$a_2, 1-\delta$} (s1)
             edge [bend left] node {$a_1, 1$} (s3);
        % (s3) edge [loop below] node {$\{a_1, a_2\}, 1$} (s3);
    \end{tikzpicture}
    \caption{}
    \label{fig:MDP-example}
    \vspace{-1.5em}
\end{wrapfigure}

Before presenting our methods, we provide an example to show that optimal values in DMDPs are generally incomparable for different delays. Depending on the stochasticity of the environment and the length of the delay, the optimal value function can be made arbitrary worse compared to that of the undelayed environment. 

Let $V^*$  be the optimal value function of the MDP sketched in Figure~\ref{fig:MDP-example} and similarly $\widetilde{V}^*$ be the optimal value function with constant observation delay of $1$. Starting at state $s_1$, the agent will receive a reward of $+1$ for taking $a_1$ in $s_1$ and $0$ otherwise. $0 \leq \delta \leq \frac{1}{2}$ controls the stochasticity of the environment. To maximize the expected return, the agent should try to stay in $s_1$ by taking $a_1$ in $s_1$ and $a_2$ in $s_2$. When there is no delay, the agent can take the appropriate action and avoid the absorbing state $s_3$. However, with delay, the agent does not observe the current state and for $0 < \delta$ it eventually ends up in $s_3$ for any policy. The ratio between the optimal values of the delayed and undelayed case can be computed (see appendix \ref{sec:mdp-example}) as
\begin{align*}
    \frac{\widetilde{V}^*(s_1)}{V^*(s_1)} = \frac{(1-\gamma)}{(1-\gamma\delta)\left(1-\gamma(1-\delta)\right)}.
\end{align*}
When $\delta=0$ the ratio is 1, while the minimal ratio of $\frac{1-\gamma}{(1-\gamma/2)^2}$ is obtained for $\delta = \frac{1}{2}$ with the ratio approaching 0 as $\gamma \rightarrow 1$. These two extreme cases correspond to the scenarios with the least and the most stochasticity in the transitions, respectively. 

In general, depending on the underlying MDP, even introducing small observation delays could downgrade the optimal policies much. By assuming smooth transition dynamics and rewards, \mbox{\cite{liotet2022delayed}} bounded this gap as a function of smoothness parameters of the underlying MDP.

%% file: sec/method.tex
\section{Delayed Control via World Models} \label{sec:method}

In this section, we begin with a seemingly simple yet crucial insight into the relationship between converting POMDPs into MDPs via world models, and translating DPOMDPs into DMDPs. This insight forms the basis for combining techniques initially developed for POMDPs and DMDPs to tackle delays in partially observable environments. Next, we elaborate on the adaptations required to incorporate delays and examine two distinct methodologies within this framework.

\subsection{World models reduce DPOMDPs to DMDPs}\label{sec:prop}

A world model, denoted by $\widehat{M}$, has two modes of operation: \emph{imagination}, where it can operate as a stateful simulator of the world with which the agent can interact (Figure \ref{fig:actor-critic}); and \emph{interaction}, where it can ground its latent state $m_t$ in real observations by sequentially incorporating their embedding into the state (Figure \ref{fig:world-model}). In some world models, imagination is implemented using the interaction mode, by replacing the real embedded observations by their reconstruction from the latent state~\citep{ha2018world}. In Dreamer, the two modes are modeled separately through \eqref{eq:z} and \eqref{eq:im} and kept equivalent in training by a dedicated loss term. In the following, we keep our discussion general by not constraining the functional form of the world model, and only requiring the two modes to be equivalent in the sense of the following definition.

\begin{defn}
    A world model $\widehat{M}$ is congruent with a POMDP $\mathcal{P}_o$ if, for any action sequence $\vec{a} = \{a_t\}_{t \geq 0}$, the stochastic process $\{o_t, m_t, r_t\}_{t \ge 0}$ induced by rolling out $\vec{a}$ in $\mathcal{P}_o$ and feeding its observations into $\widehat{M}$'s interaction mode has the same joint distribution (marginalized over $\{o_t\}_{t \ge 0}$) as the process $\{m_t, r_t\}_{t\ge0}$ induced by rolling out $\vec{a}$ in $\widehat{M}$'s imagination mode.
\end{defn}

\begin{defn}
    Given a world model $\widehat{M}$, the operation of the $d$-step delayed world model $\widehat{M}^d$ is defined as follows: in \emph{imagination} step $t$, the delayed state $m_{t-d}$ (or dummy for $t < d$) is read from $\widehat{M}$, action $a_t$ is taken in $\widehat{M}$, and the delayed reward $r_{t-d}$ (or 0 for $t < d$) is returned; in \emph{interaction} step $t \ge d$ (following $d$ dummy steps), the current observation $\tilde{o}_t = o_{t-d}$ is fed into $\widehat{M}$, the current state $\tilde{m}_t$ is read, the action $a_t$ is taken in the delayed environment but $\tilde{a}_t = a_{t-d}$ in $\widehat{M}$, and the delayed environment reward $\tilde{r}_t = r_{t-d}$ is returned.
\end{defn}

\begin{prop}\label{prop}
    If a world model $\widehat{M}$ is congruent with a POMDP $\mathcal{P}_o$, then the $d$-step delayed world model $\widehat{M}^d$ is congruent with the $d$-step delayed DPOMDP $\mathcal{P}_o^d$.
\end{prop}

\begin{proof}
    For a given action sequence $\vec{a}$, let $\{\tilde{o}_t, \tilde{r}_t\}_{t\ge 0}$ be the stochastic process induced by rolling out $\vec{a}$ in $\mathcal{P}^d_o$. In $\widehat{M}^d$'s interaction mode, the first $d$ dummy steps are skipped, and then the sequence $\{\tilde{o}_t, \tilde{a}_t\}_{t \ge d} = \{o_{t-d}, a_{t-d}\}_{t \ge d}$ is fed into $\widehat{M}$. Because this is the same process as $\widehat{M}$'s interaction with $\mathcal{P}_o$ using $\vec{a}$, it has the same distribution over $\{\tilde{m}_t, \tilde{r}_t\}_{t \ge d}$ as $\widehat{M}$'s imagination process. It remains to be verified that prepending $d$ dummy states and rewards to $\widehat{M}$'s imagination process yields $\widehat{M}^d$'s imagination process with the same action sequence, completing the proof.
\end{proof}

Proposition \ref{prop} solves a critical issue with training world models in delayed environments: training $\widehat{M}^d$ with a blackbox $\mathcal{P}_o^d$ treats the latter as a POMDP and fails to leverage its specific structure as a DPOMDP (see Sec. 3.1). Instead, we can recover $\mathcal{P}_o$ from $\mathcal{P}_o^d$ by shifting back the delays in training time, train $\widehat{M}$ for $\mathcal{P}_o$, and then Definition 2 gives us the structure of $\widehat{M}^d$ in terms of $\widehat{M}$. Proposition 3 guarantees that this process indeed models $\mathcal{P}_o^d$ correctly.

\subsection{Delay-aware training} \label{sec:delay-aware}

The learning of the world model relies on data stored in an experience replay buffer, accumulated by the agent throughout training. With delays, the storage of these data remains unaffected, as the data collection mechanism can store a transition once the subsequent observation becomes available; thus, the world model can be trained using undelayed data \(\{(o_t, a_t, r_t)\}_{t \geq 0}\). However, the distribution of collected samples is influenced by the fact that actions are selected without observing the past \(d\) states of the environment. This discrepancy leads to a divergence in the distribution of data trajectories between delayed and undelayed environments. Nevertheless, as we will see in the experiment section, the world model can still learn successfully.

In contrast to world model learning, the actor--critic component must take delays into account, as the agent is required to select $a_{t}$ based on the available information at time $t$, generated at time $t-d$. To account for delayed observations, there are two primary design strategies: either to design the policy as $\pi(\cdot|x_{t})$, conditioned on the extended state $x_t$, or condition the policy on the latent state $\hat{m}_t$ predicted from $x_{t}$. In both cases, the agent needs access to the action sequence $(a_{t-d}, ..., a_{t-1})$. Thus, we augment the experience replay buffer to store subsequent actions. In the following two sections, we will explore each of these design approaches.

\subsection{Delayed actor-critic}\label{sec:delayed-ac}

In actor--critic learning, the critic provides an estimate of the value function $V$ to aid the learning of the actor, while the actor aims to maximize the return guided by the critic's value. Since the world model is a DMDP and the extended state $x_t = (m_{t-d}, a_{t-d}, \dots, a_{t-1})$ is a state of its equivalent extended MDP, it is sufficient to condition $a_t$ on $x_t$. The critic, in contrast, only provides feedback in training time and can therefore wait to see the true state to provide a more accurate value estimate to the actor. This idea has also been explored concurrently with this work~\citep{wang2024addressing}. Thus, to directly handle observation delays in the policy, we can use the extended state to design the policy $\pi(a_t | x_t)$ along with the same critic as in the undelayed case $V(m_t)$. We refer to this method as an \textbf{Extended} actor. In practice, the policy network can be implemented with any neural architecture, such as a Multi-Layer Perceptrons (MLP), Recurrent Neural Networks~(RNN), or Transformers.

Figure \ref{fig:extended-diagram} illustrates the Extended actor diagram adapted for Dreamer. At time $t$, the agent retrieves $x_t$ from the replay buffer and performs on-policy actor--critic learning in imagination by updating the extended state with the next action. In particular, $a_t\sim\pi_\theta(a_t | x_t)$ and $x_{t+1} = (m_{t-d+1}, \{a_i\}_{i=t-d+1}^{t})$ where $m_{t-d+1}$ is the imagined latent state in \eqref{eq:im}. Note that the critic predicts the value for the current latent state $V_\psi(m_{t-d})$ while the actor outputs $a_{t}$ based on $x_t$. Thus, the imagination horizon will be increased for $d$ additional time steps since the critic provide feedback for actor's action $d$ timesteps later. The estimates of the critic and the actions of the actor are then realigned to compute the policy gradient loss function.

Another variant of the extended actor involves drawing actions from the policy $\pi(a_t | m_{t-d})$ without maintaining a memory to track previously performed actions, a concept referred to as the \textbf{Memoryless} actor. While this design choice might appear to lack the ability to capture the necessary information, the rationale behind it is that the policy $\pi(a_t | m_{t-d})$ can theoretically represent the extended state's previously performed actions within its network. This is because no new information is introduced after time $t-d$, and therefore, no additional memory is needed to store those actions.

\subsection{Latent state imagination}
Another approach is to estimate the current latent state $\hat{m}_t$, without modifying the policy architecture, and draw actions from $\pi(a_t|\hat{m}_t)$. Specifically, we can use the world model's forward dynamics in \eqref{eq:h} and \eqref{eq:im} (Figure \ref{fig:actor-critic})  for $d$ time steps, starting with $m_{t-d}$, to sample a prediction $\hat{m}_t$. Then, the agent uses $\hat{m}_t$ as the current latent state of the environment both in training and inference time. Figure \ref{fig:latent-diagram} depicts this process, which we refer to as a \textbf{Latent} actor. After computing $\hat{m}_t$, the agent performs training or policy execution the same way as in the undelayed case.

Note that estimation of the current state happens in the latent space, otherwise this approach will lead to suboptimal decisions as the agent needs to form an approximate belief over the hidden state to act optimally in the presence of delays. In other words, the agent should account for uncertainty over the true state of the environment. This also implies that the latent state should have a deterministic component to allow the agent to avoid losing information through sampling. Furthermore, Extended and Latent actors do not assume that $d$ is constant. The Latent actor can imagine for variable step lengths and Extended can employ RNNs or pad the action sequence with a special action to manage delay variability. In principle, Memoryless actors can also handle stochastic delays by conditioning the policy on $d$, though this has uncertain practical effectiveness.

%% file: sec/experiments.tex
\section{Experiments}

\subsection{Experimental setup}

\paragraph{Tasks.} To evaluate our proposed methods, we conducted experiments across a diverse set of environments. We have considered four continuous control tasks from MuJoCo~\citep{todorov2012mujoco} in Gymnasium (Gym)~\citep{towers_gymnasium_2023}: HalfCheetah-v4, HumanoidStandup-v4, Reacher-v4, and Swimmer-v4 for comparison with previous studies. Also, we extended our evaluation to six more environments from the DeepMind Control Suite (DMC) \citep{tunyasuvunakool2020}, to further examine our methods with both proprioceptive and visual observations. The distinction between these input types is critical; vector inputs provide a fully observable state of the environment, but image-based observations introduce partial observability, necessitating approaches capable of addressing delays within the POMDP framework.

\paragraph{Methods.} 
We utilized Dreamer-V3 \citep{hafner2023dreamerv3} as our primary framework\footnote{The code is available at \url{https://github.com/indylab/DelayedDreamer}.} and compare with prior studies, including D-TRPO~\citep{liotet2021learning} and DC-AC~\citep{bouteiller2020reinforcement}. We also evaluated a Dreamer-V3 agent, similar to the Latent method, except trained in an undelayed setting and tested under delayed conditions with latent imagination. This approach, referred to as \textbf{Agnostic}, can be e considered a naive use of world models to address delays. Although our methods can handle both constant and random delays, we chose to focus on fixed delays for simplicity and comparability with the baselines. While we experiment in both Gym and DMC, the baselines were originally designed and tuned for Gym environments with vector inputs, and we found the task of modifying them for image observations or tuning their hyperparameters for DMC nontrivial~\mbox{\citep{cetin2022stabilizing}} and thus outside our scope. Also, for Gym environments, we trained Dreamer variants with a budget of $500$K interactions, while D-TRPO and DC-AC trained with $5$M and $1$M environment interactions, respectively. For DMC tasks with visual inputs, we have increased the number of interactions to $1$M.

\paragraph{Architectures and hyperparameters.} 
For D-TRPO, we adopted the hyperparameters and architecture detailed in \cite{liotet2022delayed}. Similarly, for DC-AC, we replicated the hyperparameters and model from the original paper~\citep{bouteiller2020reinforcement}. In our delayed variants of Dreamer-V3, we maintained consistency by using the same set of hyperparameters and architecture as the original implementation~\citep{hafner2023dreamerv3}. The only architectural adjustment was for the extended agent, where we incorporated a Multi-Layer Perceptron (MLP) for the policy network to extend the latent state with actions. Note that while we used the same set of hyperparameters provided in the original Dreamer, we conjecture that the optimal horizon length should be smaller in both the Extended and Latent methods, where the effective horizon is longer due to the action buffer (Figures \ref{fig:extended-diagram} and \ref{fig:latent-diagram}), because accumulation of one-step errors in imagination via forward dynamics could harm the actor-critic learning part. Each experiment has been repeated with $5$ random seeds.

Note that, in all experiments, the agent will perform random actions until the first observation becomes available. While one could utilize better strategies for initial actions, using random actions is common in all existing delayed RL methods.

\subsection{Results}

\begin{figure}[!t]
    \begin{subfigure}[t]{0.33\textwidth}
        \centering
        \includegraphics[width=\textwidth]{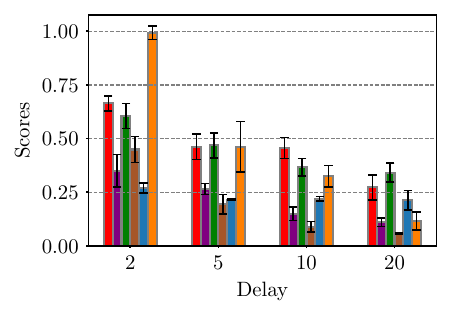}
        \vspace{-2em}
        \caption{HalfCheetah-v4}
        \label{fig:gym-halfcheetah-bars}
    \end{subfigure}
    \begin{subfigure}[t]{0.33\textwidth}
        \centering
        \includegraphics[width=\textwidth]{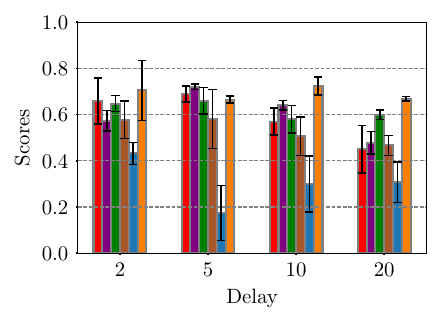}
        \vspace{-2em}
        \caption{HumanoidStandup-v4}
        \label{fig:gym-humanoidstandup-bars}
    \end{subfigure}
    \begin{subfigure}[t]{0.33\textwidth}
        \centering
        \includegraphics[width=\textwidth]{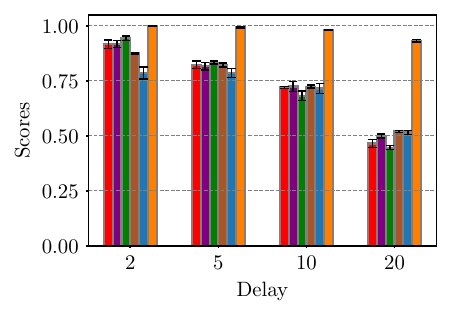}
        \vspace{-2em}
        \caption{Reacher-v4}
        \vspace{1em}
        \label{fig:gym-reacher-bars}
    \end{subfigure}

    \begin{subfigure}[b]{0.33\textwidth}
        \centering
        \includegraphics[width=\textwidth]{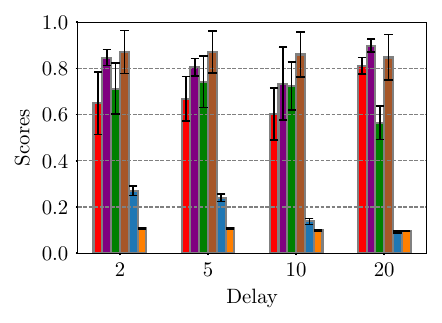}
        \vspace{-2em}
        \caption{Swimmer-v4}
        \label{fig:gym-swimmer-bars}
    \end{subfigure}
    \begin{subfigure}[b]{0.33\textwidth}
        \centering
        \includegraphics[width=\textwidth]{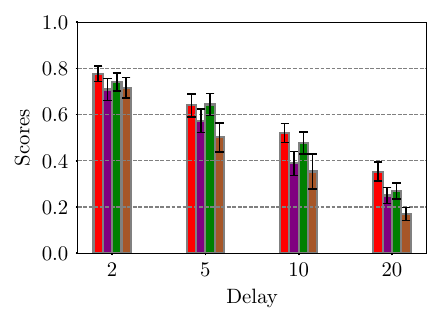}
        \vspace{-2em}
        \caption{DMC proprio}
        \label{fig:dmc-proprio-bars}
    \end{subfigure}
    \begin{subfigure}[b]{0.33\textwidth}
        \centering
        \includegraphics[width=\textwidth]{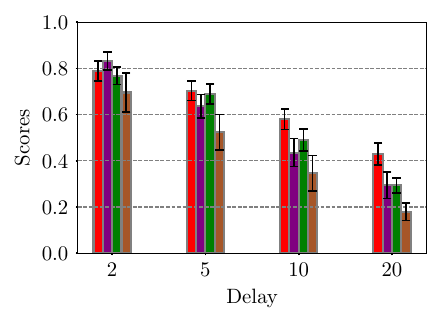}
        \vspace{-2em}
        \caption{DMC vision}
        \label{fig:dmc-vision-bars}
    \end{subfigure}
    
    \vspace{1em}
    
    \begin{subfigure}{\textwidth}
        \centering
        \includegraphics[width=0.75\textwidth]{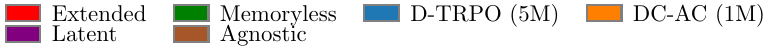}
    \end{subfigure}
    \caption{
    Normalized returns across different environments for varying delays. Bars and caps represent the mean and standard error of the mean over 5 trials, respectively. Panels (\subref{fig:dmc-proprio-bars}) and (\subref{fig:dmc-vision-bars}) are averaged over the selected suites in DMC, after normalizing the agent in the undelayed environment to 1 and the random policy to 0.
    }
\end{figure}

Figures \ref{fig:gym-halfcheetah-bars}–\ref{fig:gym-swimmer-bars} depict the results obtained from the experiments conducted on the selected Gym environments. The Dreamer variants demonstrate a significant performance improvement over \mbox{D-TRPO} across all tasks. However, DC-AC exhibits comparable performance to our methods on HalfCheetah-v4 and HumanoidStandup-v4, outperforms them on Reacher-v4, and underperforms on Swimmer-v4. One reason why our methods are not performing well on Reacher-v4 could be  underperformance of the standard Dreamer, trained and tested on the undelayed environment itself. This is evident as DC-AC achieves comparable or superior performance to the standard Dreamer on these environments with small delays. One potential explanation for this phenomenon could be the significant portion of samples in actor--critic learning that fall outside the planning horizon. This is likely due to our use of the same imagination horizon length $H=16$ and an episode length of $50$ for this task. 

Figures \ref{fig:dmc-proprio-bars} and \ref{fig:dmc-vision-bars} display the performance of our methods averaged over selected suites in DMC with proprioceptive and image inputs, respectively. Remarkably, the Agnostic method performs very well across all tasks without knowing about the delay in training time. Also, the Agnostic method needs not know the delay distribution beforehand and can be deployed on any delayed environment. However, as delay increases the performance drops more rapidly than for other methods. This is because the distribution shift between the undelayed training and delayed evaluation increases for larger delays. Similarly, the Latent method does not exhibit robustness against long delays, as we are using a one-step prediction world model. The accumulation of one-step prediction errors over longer delays causes the predicted latent state to diverge significantly from the true latent state. As expected, Extended proves to be the most robust among our variants, as it utilizes next actions and avoids the accumulation of errors present in Latent and Agnostic. Notably, Extended improves by $250\%$ on average in DMC vision tasks compared to Agnostic.

Additionally, we include training curves and tables summarizing the final test performance for all tasks in Appendix \ref{sec:exp-detail}.

\subsection{Degree of observability}\label{sec:partial-cheetah}

\begin{figure}[t!]
    \centering
    \includegraphics[width=0.45\textwidth]{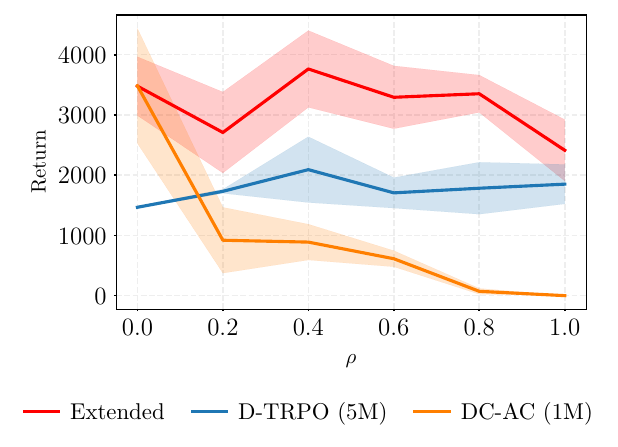}
    \caption{Return against the degree of observability in HalfCheeth-v4 for $d=5$.}
    \label{fig:partial-cheetah}
\end{figure}

Figure \ref{fig:partial-cheetah} illustrates the resilience of the Extended method and the baselines faced with an increasing level of partial observability in the HalfCheetah-v4 environment. Originally, the environment's observations encompass both the positional and the velocity information for the agent's joints. To simulate partial observability, we modified the environment to omit a $\rho$ percentage of the velocity components. The results demonstrate that both Extended and D-TRPO were capable of inferring the missing velocity components from historical observations. Although D-TRPO is specifically designed for delays in MDPs, in this particular scenario, it was able to compute the relative velocities by utilizing a transformer used to process the extended state. In contrast, DC-AC deteriorated significantly as environments became less observable.

\subsection{Takeaways}
Although world models work as a good proxy of the true environment during training, we found that applying world models naively at test time (Agnostic) is not the most effective strategy. Our experiments revealed that for shorter delays, the Memoryless and Latent approaches work the best while keeping the original architecture unchanged. However, as delays get longer they degrade the performance due to the issues of a lack of action memory and the accumulation of one-step errors issues, respectively. The Extended method, on the other hand, can maintain its performance at the expense of adding architectural complexity to the undelayed model.

%% file: sec/discussion_conclusion.tex
\section{Conclusion}

In this paper, we have proposed using world models for delayed observation within the POMDP framework. To showcase our methods, we adapted Dreamer-V3 for delay in observations and proposed two strategies, one using a delayed actor and the other latent state imagination. We discussed another version of the delayed actor which operates without action memory and additionally introduced a delay-agnostic strategy which needs not know the delay distribution beforehand. Evaluation revealed that the best of our methods, Extended, is robust to partial observability of the environment and can outperform the baselines overall, but can be sensitive to the tuning of hyperparameters. 

\section{Acknowledgement}

This work was funded in part by the National Science Foundation (Award \#2321786).

%% file: sec/appendix.tex
\section{MDP example} \label{sec:mdp-example}

The optimal policy will select $a_1$ in $s_1$ and $a_2$ in $s_2$. Then, by Bellman optimal equation for the state-value function we have
\begin{align}
    &V^*(s_1) = 1 + \gamma (1-\delta) V^*(s_1) + \gamma \delta V^*(s_2) \\
    &V^*(s_2) = \gamma (1-\delta) V^*(s_1) + \gamma \delta V^*(s_2),
\end{align}
which yields $V^*(s_1) = \frac{1 - \gamma \delta}{1 - \gamma}$. In the case of observation delay with $d=1$, the optimal policy will select $a_1$ in every state since $\frac{1}{2} \leq \delta$ and thus,
\begin{align}
    \widetilde{V}^*(x_1) = (1-\delta) \left(1 + \gamma \widetilde{V}^*(x_1)\right),
\end{align}
where $x_1$ is the extended state $(s_1, a_1)$. For the current state of the environment in $s_1$, we have $\widetilde{V}^*(s_1) = 1 + \gamma \widetilde{V}^*(x_1)$. Therefore, $\widetilde{V}^*(s_1) = \frac{1}{1-\gamma(1-\delta)}$.

\section{Experiment Details} \label{sec:exp-detail}

In this section, we have included the training curves and the final results of our experiments across the selected environments in Gym and DMC. In order to have a fair comparison between the methods, we have used the same random seed for generating random actions at the beginning of the episode. We refer to DMC tasks with proprioceptive and image observations as DMC proprio and vision, respectively.

\subsection{Gym results}
\input{sec/tables/gym_proprio}
\subsection{DMC proprio results}
\input{sec/tables/dmc_proprio}
\subsection{DMC vision results}
\input{sec/tables/dmc_vision}

\clearpage
\subsection{Gym training curves}
\begin{figure}[htbp]
    \centering
    \includegraphics[width=\textwidth]{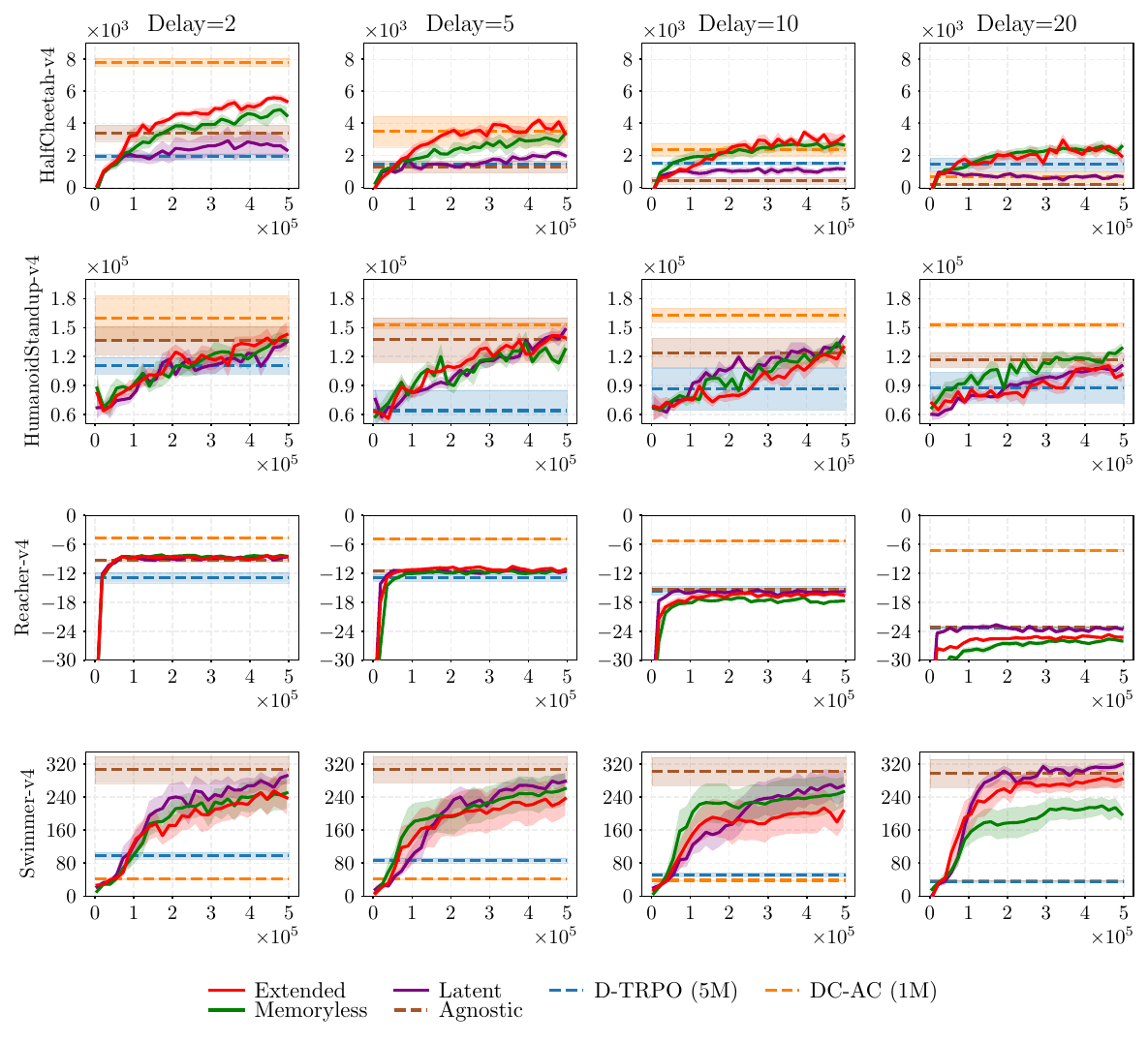}
    \caption{Training curves for the set of tasks in Gym. Dreamer variants trained with 500K interactions of the environment, while D-TRPO and DC-AC used 5M and 1M interactions, respectively. For D-TRPO and DC-AC, we have plotted the final training performance.}
    \label{fig:gym-training-1}
\end{figure}

\subsection{DMC proprio training curves}
\begin{figure}[H]
    \centering
    \includegraphics[width=0.99\textwidth]{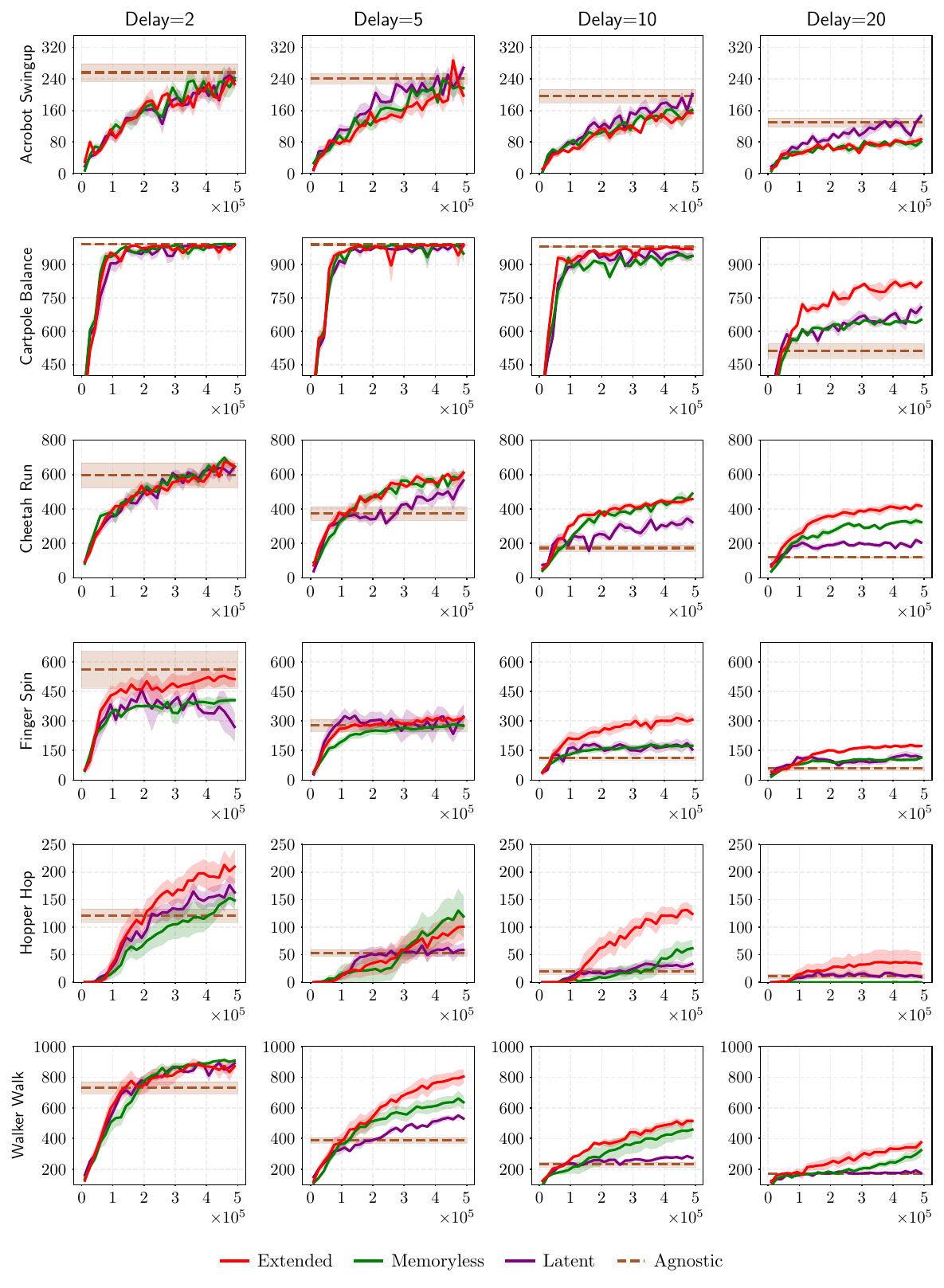}
    \caption{Training curves for the set of tasks in DMC with proprioceptive inputs with 500K interactions.}
    \label{fig:dmc-proprio-training}
\end{figure}

\subsection{DMC vision training curves}
\begin{figure}[H]
    \centering
    \includegraphics[width=.99\textwidth]{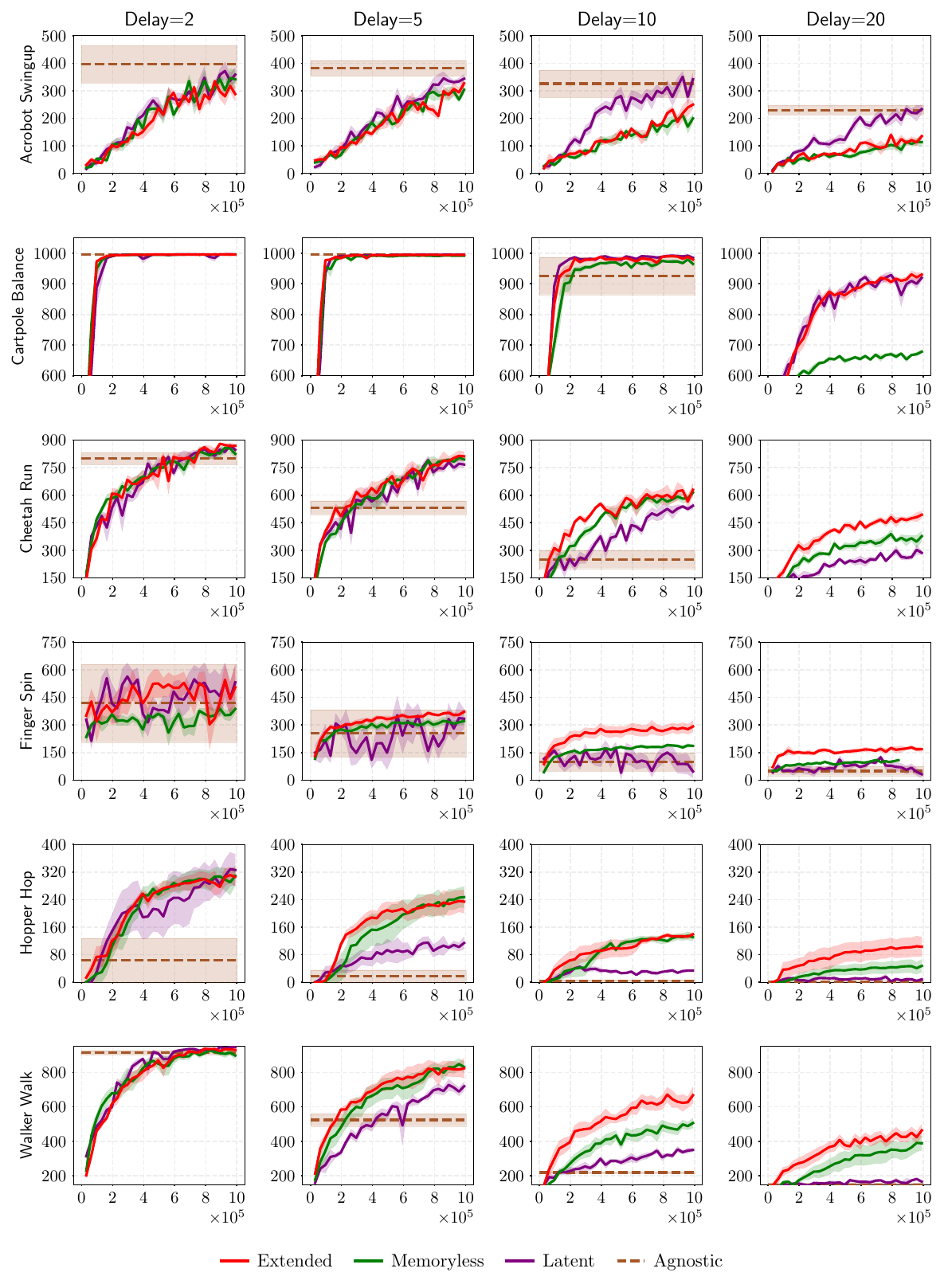}
    \caption{Training curves for the set of tasks in DMC with visual inputs with 1M interactions.}
    \label{fig:dmc-vision-training}
\end{figure}

% \section{World Model Loss} \label{sec:world-model-loss}

%% file: sec/tables/gym_proprio.tex
\newcolumntype{H}{>{\setbox0=\hbox\bgroup}c<{\egroup}@{}}

\begin{table}[h]
    \begin{center}
        \resizebox{\textwidth}{!}{
        \begin{tabular}{|c|c|c|c|c|c|c|c|}
        \hline
        \textbf{Task} & \textbf{Delay} & \makecell{\textbf{Extended}} & \makecell{\textbf{Memoryless}} & \makecell{\textbf{Latent}} & \makecell{\textbf{Agnostic}} & \textbf{D-TRPO} & \textbf{DC-AC}\\
        \hline
        \multirow{4}{*}{\makecell{HalfCheetah-v4 \\ ($\times 10^3$)}}
        & 2 & $5.13 \pm 0.28$ & $4.65 \pm 0.47$ & $2.57 \pm 0.61$ & $3.37 \pm 0.49$ & $1.91 \pm 0.18$ & $\mathbf{7.82 \pm 0.25}$\\
        & 5 & $3.47 \pm 0.48$ & $\mathbf{3.53 \pm 0.47}$ & $1.87 \pm 0.20$ & $1.28 \pm 0.37$ & $1.46 \pm 0.01$ & $3.48 \pm 0.95$\\
        & 10 & $\mathbf{3.43 \pm 0.39}$ & $2.70 \pm 0.33$ & $0.93 \pm 0.25$ & $0.42 \pm 0.19$ & $1.49 \pm 0.08$ & $2.35 \pm 0.41$\\
        & 20 & $1.93 \pm 0.47$ & $\mathbf{2.49 \pm 0.35}$ & $0.60 \pm 0.16$ & $0.15 \pm 0.02$ & $1.44 \pm 0.37$ & $0.65 \pm 0.34$\\
        \hline
        \multirow{4}{*}{\makecell{HumanoidStandup-v4 \\ ($\times 10^5$)}}
        & 2 & $1.51 \pm 0.18$ & $1.49 \pm 0.06$ & $1.36 \pm 0.08$ & $1.37 \pm 0.15$ & $1.11 \pm 0.09$ & $\mathbf{1.60 \pm 0.23}$\\
        & 5 & $1.57 \pm 0.06$ & $1.52 \pm 0.10$ & $\mathbf{1.63 \pm 0.02}$ & $1.37 \pm 0.23$ & $0.64 \pm 0.21$ & $1.52 \pm 0.03$\\
        & 10 & $1.35 \pm 0.11$ & $1.37 \pm 0.11$ & $1.48 \pm 0.04$ & $1.24 \pm 0.15$ & $0.87 \pm 0.22$ & $\mathbf{1.63 \pm 0.07}$\\
        & 20 & $1.14 \pm 0.19$ & $1.41 \pm 0.04$ & $1.19 \pm 0.09$ & $1.17 \pm 0.08$ & $0.88 \pm 0.16$ & $\mathbf{1.53 \pm 0.02}$\\
        \hline
        \multirow{4}{*}{Reacher-v4}
        & 2 & $-7.8 \pm 0.7$ & $-6.8 \pm 0.4$ & $-7.8 \pm 0.6$ & $-9.4 \pm 0.1$ & $-12.9 \pm 1.1$ & $\mathbf{-4.6 \pm 0.0}$\\
        & 5 & $-11.4 \pm 0.6$ & $-11.0 \pm 0.2$ & $-11.7 \pm 0.7$ & $-11.5 \pm 0.4$ & $-12.9 \pm 0.8$ & $\mathbf{-4.8 \pm 0.1}$\\
        & 10 & $-15.4 \pm 0.2$ & $-16.8 \pm 0.8$ & $-15.2 \pm 0.9$ & $-15.3 \pm 0.3$ & $-15.6 \pm 0.9$ & $\mathbf{-5.3 \pm 0.1}$\\
        & 20 & $-25.3 \pm 0.7$ & $-26.0 \pm 0.3$ & $-23.9 \pm 0.3$ & $-23.1 \pm 0.2$ & $-23.3 \pm 0.4$ & $\mathbf{-7.2 \pm 0.2}$\\
        \hline
        \multirow{4}{*}{Swimmer-v4}
        & 2 & $229.2 \pm 46.8$ & $251.2 \pm 38.5$ & $297.5 \pm 12.0$ & $\mathbf{305.9 \pm 32.3}$ & $98.3 \pm 7.1$ & $42.1 \pm 1.0$\\
        & 5 & $236.0 \pm 33.6$ & $261.5 \pm 38.6$ & $283.3 \pm 13.1$ & $\mathbf{306.3 \pm 31.5}$ & $87.5 \pm 5.7$ & $41.5 \pm 0.7$\\
        & 10 & $213.1 \pm 39.2$ & $255.0 \pm 36.3$ & $258.6 \pm 54.6$ & $\mathbf{302.3 \pm 33.7}$ & $52.1 \pm 4.4$ & $38.6 \pm 0.9$\\
        & 20 & $285.1 \pm 12.4$ & $200.2 \pm 25.0$ & $\mathbf{315.5 \pm 9.7}$ & $297.8 \pm 34.2$ & $36.0 \pm 1.4$ & $37.2 \pm 0.2$\\
        \hline
        \end{tabular}}
        \caption{Final test returns on tasks in Gym. Results are presented as the mean $\pm$ standard error of the mean.}
        \label{tbl:gym_proprio}
    \end{center}
\end{table}

%% file: sec/tables/dmc_proprio.tex
\begin{table}[H]
    \begin{center}
        \resizebox{0.7\textwidth}{!}{
        \begin{tabular}{|c|c|c|c|c|c|}
        \hline
        \textbf{Task} & \textbf{Delay} & \makecell{\textbf{Extended}} & \makecell{\textbf{Memoryless}} & \makecell{\textbf{Latent}} & \makecell{\textbf{Agnostic}}\\
        \hline
        \multirow{3}{*}{Acrobot Swingup} 
           & 2 & $223.0 \pm 9.5$ & $234.8 \pm 26.0$ & $224.4 \pm 19.7$ & $\mathbf{256.6 \pm 22.9}$\\
           & 5 & $208.5 \pm 42.1$ & $228.6 \pm 32.6$ & $\mathbf{257.9 \pm 36.2}$ & $240.9 \pm 13.1$\\
           & 10 & $120.7 \pm 12.8$ & $159.7 \pm 15.2$ & $193.2 \pm 30.3$ & $\mathbf{197.0 \pm 17.0}$\\
           & 20 & $78.7 \pm 15.0$ & $93.7 \pm 15.6$ & $\mathbf{144.0 \pm 14.5}$ & $129.7 \pm 11.3$\\
        \hline
        \multirow{3}{*}{Cartpole Balance} 
           & 2 & $986.7 \pm 4.2$ & $\mathbf{993.2 \pm 0.4}$ & $990.4 \pm 1.0$ & $992.4 \pm 2.4$\\
           & 5 & $\mathbf{990.6 \pm 1.1}$ & $986.5 \pm 2.0$ & $985.0 \pm 3.9$ & $990.0 \pm 3.4$\\
           & 10 & $966.4 \pm 4.9$ & $948.0 \pm 12.0$ & $916.7 \pm 48.7$ & $\mathbf{980.4 \pm 2.1}$\\
           & 20 & $\mathbf{799.8 \pm 32.5}$ & $664.8 \pm 15.0$ & $699.9 \pm 40.0$ & $511.9 \pm 32.8$\\
        \hline
        \multirow{3}{*}{Cheetah Run} 
           & 2 & $632.9 \pm 25.3$ & $666.3 \pm 22.7$ & $\mathbf{668.9 \pm 30.0}$ & $596.1 \pm 72.8$\\
           & 5 & $600.2 \pm 24.8$ & $\mathbf{648.3 \pm 17.4}$ & $534.9 \pm 54.7$ & $373.2 \pm 39.5$\\
           & 10 & $446.7 \pm 20.1$ & $\mathbf{496.1 \pm 6.6}$ & $304.0 \pm 60.3$ & $173.6 \pm 20.6$\\
           & 20 & $\mathbf{418.0 \pm 16.9}$ & $337.3 \pm 14.2$ & $204.4 \pm 11.7$ & $119.7 \pm 3.4$\\
        \hline
        \multirow{3}{*}{Finger Spin} 
           & 2 & $504.0 \pm 21.1$ & $426.7 \pm 18.4$ & $253.4 \pm 83.0$ & $\mathbf{561.1 \pm 93.5}$\\
           & 5 & $303.6 \pm 14.8$ & $279.8 \pm 17.8$ & $\mathbf{310.0 \pm 57.0}$ & $278.3 \pm 31.1$\\
           & 10 & $\mathbf{307.0 \pm 23.3}$ & $171.5 \pm 12.4$ & $153.7 \pm 22.1$ & $111.1 \pm 7.5$\\
           & 20 & $\mathbf{183.9 \pm 13.6}$ & $116.3 \pm 4.6$ & $118.4 \pm 11.0$ & $59.5 \pm 8.1$\\
        \hline
        \multirow{3}{*}{Hopper Hop} 
           & 2 & $\mathbf{206.6 \pm 31.7}$ & $143.5 \pm 15.9$ & $164.4 \pm 22.3$ & $121.0 \pm 11.9$\\
           & 5 & $99.3 \pm 23.4$ & $\mathbf{130.8 \pm 37.4}$ & $54.1 \pm 11.4$ & $53.6 \pm 6.2$\\
           & 10 & $\mathbf{112.9 \pm 16.4}$ & $64.6 \pm 15.4$ & $33.9 \pm 8.5$ & $20.0 \pm 4.7$\\
           & 20 & $\mathbf{31.8 \pm 19.8}$ & $0.0 \pm 0.0$ & $7.7 \pm 1.8$ & $11.1 \pm 3.6$\\
        \hline
        \multirow{3}{*}{Walker Walk} 
           & 2 & $877.9 \pm 17.2$ & $\mathbf{916.1 \pm 6.7}$ & $873.5 \pm 29.8$ & $732.5 \pm 39.1$\\
           & 5 & $\mathbf{789.5 \pm 50.2}$ & $622.9 \pm 47.2$ & $530.9 \pm 22.2$ & $389.5 \pm 18.0$\\
           & 10 & $\mathbf{518.9 \pm 16.9}$ & $461.8 \pm 44.9$ & $278.0 \pm 13.8$ & $230.6 \pm 14.3$\\
           & 20 & $\mathbf{381.6 \pm 19.3}$ & $339.5 \pm 43.1$ & $173.2 \pm 7.7$ & $170.3 \pm 0.2$\\
        \hline
        \end{tabular}}
    \end{center}
    \caption{Final test returns on tasks in DMC with proprioceptive inputs. Results are presented as the mean $\pm$ standard error of the mean.}
    \label{tbl:dmc_proprio}
\end{table}

%% file: sec/tables/dmc_vision.tex
\begin{table}[H]
    \begin{center}
        \resizebox{0.7\textwidth}{!}{
        \begin{tabular}{|c|c|c|c|c|c|}
        \hline
        \textbf{Task} & \textbf{Delay} & \makecell{\textbf{Extended}} & \makecell{\textbf{Memoryless}} & \makecell{\textbf{Latent}} & \makecell{\textbf{Agnostic}}\\
        \hline
        \multirow{3}{*}{Acrobot Swingup} 
           & 2 & $300.8 \pm 32.6$ & $355.4 \pm 38.8$ & $374.9 \pm 35.0$ & $\mathbf{396.3 \pm 67.3}$\\
           & 5 & $301.1 \pm 16.8$ & $284.6 \pm 26.5$ & $336.3 \pm 26.8$ & $\mathbf{382.6 \pm 28.5}$\\
           & 10 & $278.7 \pm 11.7$ & $191.0 \pm 18.2$ & $\mathbf{328.0 \pm 28.0}$ & $326.0 \pm 49.3$\\
           & 20 & $139.4 \pm 7.9$ & $108.1 \pm 21.3$ & $212.9 \pm 34.7$ & $\mathbf{230.4 \pm 16.5}$\\
        \hline
        \multirow{3}{*}{Cartpole Balance} 
           & 2 & $994.7 \pm 1.3$ & $996.0 \pm 0.2$ & $\mathbf{996.3 \pm 0.1}$ & $\mathbf{996.3 \pm 0.1}$\\
           & 5 & $994.6 \pm 0.1$ & $992.6 \pm 0.3$ & $994.8 \pm 0.5$ & $\mathbf{995.2 \pm 0.1}$\\
           & 10 & $975.7 \pm 5.6$ & $953.8 \pm 12.5$ & $\mathbf{979.9 \pm 5.3}$ & $926.0 \pm 60.7$\\
           & 20 & $928.0 \pm 11.6$ & $681.3 \pm 4.2$ & $\mathbf{935.3 \pm 8.4}$ & $570.8 \pm 27.0$\\
        \hline
        \multirow{3}{*}{Cheetah Run} 
           & 2 & $871.7 \pm 11.3$ & $839.8 \pm 28.9$ & $\mathbf{873.0 \pm 4.1}$ & $799.1 \pm 32.0$\\
           & 5 & $\mathbf{816.2 \pm 20.9}$ & $812.1 \pm 13.0$ & $783.9 \pm 44.3$ & $533.1 \pm 37.3$\\
           & 10 & $\mathbf{640.0 \pm 32.9}$ & $610.9 \pm 16.4$ & $542.2 \pm 24.8$ & $251.2 \pm 49.5$\\
           & 20 & $\mathbf{493.3 \pm 21.8}$ & $395.5 \pm 21.2$ & $273.1 \pm 24.6$ & $118.8 \pm 11.4$\\
        \hline
        \multirow{3}{*}{Finger Spin} 
           & 2 & $521.6 \pm 123.3$ & $391.4 \pm 25.2$ & $\mathbf{576.7 \pm 117.7}$ & $418.0 \pm 211.1$\\
           & 5 & $\mathbf{374.9 \pm 11.3}$ & $321.7 \pm 19.3$ & $316.7 \pm 34.6$ & $254.4 \pm 128.8$\\
           & 10 & $\mathbf{291.7 \pm 30.9}$ & $186.9 \pm 7.7$ & $53.0 \pm 35.3$ & $99.7 \pm 50.1$\\
           & 20 & $\mathbf{172.2 \pm 11.9}$ & $108.3 \pm 3.4$ & $33.8 \pm 22.8$ & $49.3 \pm 24.2$\\
        \hline
        \multirow{3}{*}{Hopper Hop} 
           & 2 & $304.6 \pm 27.2$ & $313.2 \pm 27.4$ & $\mathbf{325.2 \pm 41.5}$ & $64.0 \pm 64.0$\\
           & 5 & $232.5 \pm 33.6$ & $\mathbf{246.7 \pm 32.3}$ & $114.8 \pm 25.9$ & $18.0 \pm 17.8$\\
           & 10 & $\mathbf{136.7 \pm 6.3}$ & $131.3 \pm 13.4$ & $33.3 \pm 4.9$ & $3.5 \pm 3.5$\\
           & 20 & $\mathbf{103.0 \pm 28.3}$ & $45.3 \pm 18.1$ & $4.4 \pm 1.3$ & $2.3 \pm 2.3$\\
        \hline
        \multirow{3}{*}{Walker Walk} 
           & 2 & $932.0 \pm 9.6$ & $895.6 \pm 27.7$ & $\mathbf{944.7 \pm 8.1}$ & $916.0 \pm 12.8$\\
           & 5 & $\mathbf{821.8 \pm 60.8}$ & $819.2 \pm 33.8$ & $718.7 \pm 30.1$ & $524.4 \pm 35.6$\\
           & 10 & $\mathbf{657.1 \pm 56.2}$ & $499.7 \pm 23.4$ & $344.4 \pm 18.0$ & $219.4 \pm 5.5$\\
           & 20 & $\mathbf{474.4 \pm 39.2}$ & $399.0 \pm 38.3$ & $166.3 \pm 3.1$ & $147.3 \pm 11.2$\\
        \hline
        \end{tabular}}
    \end{center}
    \caption{Final test returns on tasks in DMC with visual inputs. Results are presented as the mean $\pm$ standard error of the mean.}
    \label{tbl:dmc_vision}
\end{table}